\documentclass[preprint,12pt]{elsarticle}

\usepackage{amssymb}
\usepackage{amsmath}
\usepackage{amsthm}

\usepackage{hyperref}
\usepackage{url}
\usepackage{amsfonts}
\usepackage{graphicx}
\usepackage{enumitem}
\usepackage{booktabs}
\usepackage{tikz}
\usepackage{multirow}
\usepackage{threeparttable}
\usepackage{tabularx}
\usepackage{array}
\usepackage{listings}
\usepackage[dvipsnames]{xcolor}
\usepackage{microtype}
\usepackage{subcaption}

\newcommand{\act}{\textup{act}}

\newtheorem{theorem}{Theorem}
\newtheorem{proposition}{Proposition}

\theoremstyle{definition}
\newtheorem{definition}{Definition}
\theoremstyle{remark}
\newtheorem{remark}{Remark}
\newtheorem{example}{Example}
\usetikzlibrary{positioning,calc,shapes.geometric,arrows.meta}

\tikzset{
  neuron/.style = {circle, draw, minimum size=14pt, inner sep=1pt},
  >=Stealth, shorten >=1pt
}
\lstdefinestyle{pythoncode}{
    basicstyle=\fontsize{9}{9}\tt,
    captionpos=b,                   
    commentstyle=\color{OliveGreen},    
    deletekeywords={...},            
    escapeinside={\%*}{*)},         
    extendedchars=true,              
    frame=lines,
    breaklines=true,
    keepspaces=true,              
    keywordstyle=\color{blue},  
    language=python,        
    morekeywords={*,...},         
    numbers=left,         
    numbersep=5pt,    
    numberstyle=\tiny\color{black}, 
    xleftmargin=12pt,
    rulecolor=\color{black},   
    showspaces=false,              
    showstringspaces=false,
    showtabs=false, 
    stepnumber=1,    
    stringstyle=\color{red}, 
    tabsize=2,      
    title=\lstname,     
    backgroundcolor=\color{lightgray}
}
\usepackage[nolist]{acronym}

\begin{acronym}[XXX] 
\acro{ai}[AI]{Artificial Intelligence}
\acro{auc}[AUC]{Area Under the Curve}
\acro{cnn}[CNN]{Convolutional Neural Network}
\acro{dl}[DL]{Deep Learning}
\acro{dnn}[DNN]{Deep Neural Network}
\acro{gelu}[GELU]{Gaussian Error Linear Unit}
\acro{glai}[GLAI]{GreenLightningAI}
\acro{kd}[KD]{Knowledge Distillation}
\acro{ml}[ML]{Machine Learning}
\acro{mlp}[MLP]{Multilayer Perceptron}
\acro{mva}[MVA]{Maximum Validation Accuracy}
\acro{mha}[MHA]{Multihead Attention}
\acro{moe}[MoE]{Mixture-of-Experts}
\acro{mimick}[NM]{Neural Mimicking}
\acro{nlp}[NLP]{Natural Language Processing}
\acro{nn}[NN]{Neural Network}
\acro{oui}[OUI]{Overfitting-Underfitting Indicator}
\acro{prelu}[PReLU]{Parametric Rectified Linear Unit}
\acro{poc}[PoC]{Proof-of-Concept}
\acro{relu}[ReLU]{Rectified Linear Unit}
\acro{silu}[SiLU]{Sigmoid Linear Unit}
\acro{sgd}[SGD]{Stochastic Gradient Descent}
\acro{svd}[SVD]{Singular Value Decomposition}
\acro{sota}[SOTA]{state-of-the-art}
\acro{vt}[VT]{Vision Transformer}
\acro{vit}[ViT]{Vision Transformer}
\acro{wd}[WD]{Weight Decay}
\end{acronym}

\makeatletter
\newcommand{\myauthmarks}[1]{%
  \ensuremath{%
    \ifcase#1\or
      \relax 
    \or\|
    \else\@ctrerr
    \fi}}
\AtBeginDocument{}
\makeatother

\journal{Neural Networks}

\begin{document}

\begin{frontmatter}



\title{GLAI: GreenLightningAI for Accelerated Training through Knowledge Decoupling} 


\author[1]{Jose I. Mestre}\ead{jimesmir@disca.upv.es}
\author[1]{Alberto Fernández-Hernández}\ead{a.fernandez@upv.es}
\author[1]{Cristian Pérez-Corral}\ead{cpercor@upv.es}
\author[2]{Manuel F. Dolz}\ead{dolzm@uji.es}
\author[3]{Jose Duato}\ead{jose.duato@openchip.com}
\author[1]{Enrique S. Quintana-Ortí}\ead{quintana@disca.upv.es}

\fntext[eq]{All authors contributed equally to this work.}

\affiliation[1]{organization={Universitat Politècnica de València},
            city={Valencia},
            country={Spain}}
\affiliation[2]{organization={Universitat Jaume I},
            city={Castellón de la Plana},
            country={Spain}}
\affiliation[3]{organization={Openchip \& Software Technologies},
            city={Barcelona},
            country={Spain}}

\begin{abstract}
In this work we introduce GreenLightningAI (GLAI), a new architectural block designed as an alternative to conventional Multilayer Perceptrons (MLPs). The central idea is to separate two types of knowledge that are usually entangled during training: (i) \textit{structural knowledge}, encoded by the stable activation patterns induced by Rectified Linear Unit (ReLU) activations; and (ii) \textit{quantitative knowledge}, carried by the numerical weights and biases. By fixing the structure once stabilized, GLAI reformulates the MLP as a combination of paths, where only the quantitative component is optimized. This reformulation preserves the universal approximation capabilities of MLPs, while providing statistically significant evidence of faster training than the MLPs it replaces, with an average wall-clock speedup of $1.92\times$ across the six experimental settings. Crucially, GLAI is not just another classifier, but a generic block with the potential to replace MLPs for some tasks, from supervised heads with frozen backbones to projection layers in self-supervised learning or few-shot classifiers. Across diverse experimental setups, GLAI consistently matches or exceeds the accuracy of MLPs with an equivalent number of parameters, while reducing training time on average. Overall, GLAI establishes a new design principle that opens a direction for future integration into large-scale architectures such as Transformers, where MLP blocks dominate the computational footprint.
\end{abstract}


\begin{highlights}
\item Introduces GLAI for fast retraining by decoupling routing and weight updates.
\item Models ReLU MLPs using paths of active neurons with fixed activation decisions.
\item Retrains only the path weights, avoiding full weight re-optimization.
\item Uses a compact estimator to select active neuron paths during retraining.
\item Enables efficient fine-tuning as a drop-in replacement in common workflows.
\end{highlights}


\end{frontmatter}


\section{Introduction}\label{sec:introduction}

\acp{mlp} have consistently remained at the core of modern \ac{dl} architectures. From the early theoretical foundations in the late 20th century \citep{HORNIK1989359, cybenko1989approximation, hornik1991approximation}, to the emergence of recurrent networks and the introduction of LSTMs for sequence modeling \citep{hochreiter1997long}, the breakthrough of convolutional networks for computer vision \citep{lecun1998gradient, krizhevsky2012imagenet}, the dominance of Transformers across modalities \citep{vaswani2017attention, devlin2019bert, brown2020language, vit}, and the recent rise of sparsely-gated \ac{moe} architectures \citep{shazeer2017outrageouslylargeneuralnetworks}, \acp{mlp} have persisted as a fundamental building block.  This endurance is explained by its strong expressive capacity: \acp{mlp} are universal approximators of nonlinear functions, capable of representing arbitrarily complex mappings. A central aspect of this expressivity arises from the combinatorial structure induced by activation functions such as \ac{relu}, which partition the input space into a collection of regions, each one defining a linear relationship between inputs and outputs, determined by binary activation patterns \citep{montufar2014number, raghu2017expressive}.  

Despite their central role, the training of \ac{mlp} modules remains both costly and opaque. To clarify this process, a conceptual distinction between two forms of knowledge can be established. We revisit this distinction here and develop it mathematically, using it as the foundation on which the posterior framework is built. The first is \emph{structural knowledge}, referring to the discrete activation patterns that determine how information flows through the network. The second is \emph{quantitative knowledge}, which refers to the numerical values generated by each neuron and subsequently propagated and combined. The same study further showed that structural knowledge converges much earlier than its quantitative counterpart: activation patterns stabilize after relatively few training epochs, whereas the numerical outputs continue to evolve as the weights adjust over longer timescales. This observation motivates the possibility of decoupling the two components, freezing the structural part once it has converged, and retraining only the quantitative one.

Building directly on this theoretical result, we present \acf{glai}, the first architecture to operationalize this principle. \ac{glai} is designed as a drop-in replacement for an \ac{mlp}: it preserves equivalent representational power while enabling substantially faster training. The core idea is simple: a reduced-size \ac{mlp} is trained until structural knowledge has converged; its activation patterns are then frozen, transforming the network into a fixed piecewise-linear system. At this point, the model can be re-expressed as a linear operator over all active paths, which can be efficiently re-trained. By fixing a sufficiently mature structural representation, \ac{glai} guarantees expressivity while dramatically accelerating the quantitative optimization.  

Our contributions can be summarized as follows:
\begin{enumerate}
    \item We introduce \ac{glai}, an architectural paradigm that replaces conventional \ac{relu}-based \acp{mlp} with an equivalent formulation of comparable parameter count. We provide formal foundations, including proofs that any \ac{mlp} can be re-expressed as a \ac{glai} model without loss of expressive power.  
    \item We empirically show that decoupling structural from quantitative knowledge allows \ac{glai} to optimize more efficiently, converging in fewer update steps while matching or even surpassing the accuracy of standard \acp{mlp}.
    \item We validate the generality of \ac{glai} in settings where \acp{mlp} play a central role: supervised heads on frozen backbones for classification, projection layers in self-supervised learning, and few-shot adaptation. These scenarios demonstrate both scalability across model widths and depths, and practical utility in domains such as computer vision and language processing.  
\end{enumerate}
This work positions \ac{glai} as a new design principle for feedforward components, rather than as a task-specific model. While our present focus is on replacing isolated \acp{mlp}, the formulation may provide a basis for future exploration in larger architectures such as Transformers, where stacked \acp{mlp} account for a substantial share of computation. Assessing how \ac{glai} can be effectively integrated into such settings remains an open direction for further research. 

The remainder of the paper is organized as follows. Section \ref{sec:related_work} reviews related work in the areas of network structure, interpretability, and training acceleration. Section \ref{sec:theoretical_framework} introduces the theoretical foundations of our approach, including formal definitions and the mathematical structure of the \ac{glai} framework. Section \ref{sec:experimental_results} presents the experimental evaluation across diverse use cases and configurations. Finally, Section \ref{sec:conclusions} concludes the paper and discusses future research directions.

\section{Related Work}\label{sec:related_work}

The foundations of \ac{glai} build upon prior work by \cite{duato}, which introduced the formal separation of structural and quantitative knowledge in \ac{relu}-based \acp{dnn} and demonstrated that activation patterns stabilize well before the numerical parameters converge. Some references regarding the study of metrics designed to assess changes in activation patterns include the works by \cite{hartman} and \cite{oui}, which evaluate the variability of structural knowledge and provide further evidence of its convergence during training. 

This perspective connects naturally with earlier analyses of \ac{relu}-based \acp{mlp} as piecewise linear functions. Work by \cite{montufar2014number} demonstrated the exponential growth of linear regions with depth, while \cite{hanin2019complexitylinearregionsdeep} showed that practical \acp{mlp} typically operate in far fewer regions. Such results suggest that the effective complexity of a trained \ac{mlp} is lower than its theoretical capacity. \ac{glai} leverages this observation by intervening once the \ac{mlp} head has implicitly committed to a stable subset of linear regions sufficient for the task.

Several authors have further developed path- and region-based views of \ac{relu}-based \acp{mlp}. For instance, \cite{meng2021mathcalgsgdoptimizingreluneural} proposed the $\mathcal{G}$-space framework, optimizing directly over active paths, while \cite{sudjianto2020unwrappingblackboxdeep} introduced Aletheia to interpret networks by decomposing them into local linear models. These works highlight the explanatory and computational value of activation patterns and paths. \ac{glai} builds on similar conceptual foundations but shifts the emphasis: instead of optimizing in a transformed parameter space or prioritizing interpretability, our approach treats paths as the central design element, yielding a novel architecture in which structural knowledge becomes fixed once its stabilization is achieved.

Efficiency gains in \ac{dl} have been pursued through both partial training and parameter-efficient adaptation. Methods such as FreezeOut \citep{brock2017freezeoutacceleratetrainingprogressively}, progressive freezing \citep{NEURIPS2022_794a425a}, and greedy layer-wise strategies \citep{belilovsky2019greedylayerwiselearningscale} show that computation can be reduced by freezing stable components of the backbone without sacrificing accuracy. In parallel, the transfer learning literature has introduced parameter-efficient techniques that add small modules while leaving most of the backbone untouched, including adapters \citep{houlsby2019adapter}, low-rank updates (\textsc{LoRA}, \citealp{hu2022lora}), bias-only tuning (\textsc{BitFit}, \citealp{benzaken2022bitfit}), and prompt-based methods \citep{lester2021prompt,li2021prefixtuning} in NLP, as well as Side-Tuning \citep{zhang2020sidetuning}, Visual Prompt Tuning \citep{jia2022vpt}, and AdaptFormer \citep{chen2022adaptformer} in vision. While these approaches are effective, they all intervene directly in the backbone by modifying or extending its architecture. In contrast, \ac{glai} follows the same philosophy of exploiting early stabilization for efficiency, but does so exclusively at the head level, leaving the pretrained backbone intact and providing an orthogonal path to resource savings.

Several learning settings highlight that the head plays a decisive role, much like in standard fine-tuning. In self-supervised representation learning, the architecture of the head is central: SimCLR demonstrated that a projection head is essential for disentangling invariances \citep{chen2020simclr}, while BYOL \citep{grill2020byol} and SimSiam \citep{chen2021simsiam} relied on predictors to stabilize training and improve downstream utility. Similarly, in few-shot learning, lightweight episodic classifiers such as Matching Networks \citep{vinyals2016matching}, Prototypical Networks \citep{snell2017protonets}, Relation Networks \citep{sung2018relation}, and MAML \citep{finn2017maml} enable rapid adaptation to novel classes with very limited data. Although these approaches pursue goals distinct from efficiency, they converge on the idea that head design is decisive for generalization. \ac{glai} builds directly on this insight, providing a structured replacement for conventional \acp{mlp} that preserves accuracy while accelerating training, thereby extending the benefits of careful head design beyond specialized regimes to standard transfer learning scenarios.

Finally, lightweight protocols such as linear probing \citep{oord2018representation,caron2021emerging} and angular classifiers like ArcFace \citep{deng2019arcface} illustrate that even simple heads can provide valuable insights into representation quality or improve class separability without modifying the backbone. These methods are computationally inexpensive and therefore serve as practical lower bounds in transfer learning pipelines, but they typically fall short of the accuracy attainable with a full \ac{mlp} head. \ac{glai} builds upon this perspective by offering a head that remains efficient while matching the validation scores of conventional \acp{mlp}, thus surpassing the limitations of purely lightweight alternatives.

\section{Theoretical Framework}\label{sec:theoretical_framework}

This section introduces the mathematical framework underlying the \ac{glai} framework. We begin by formalizing the notions of structural and quantitative knowledge through the representation of activation paths, showing that any \ac{mlp} with \ac{relu} activations can be reformulated by separating these two components. Building on this result, we define \ac{glai} as an alternative model and demonstrate that every \ac{mlp} can be equivalently represented in this form. This section also includes other theoretical analysis like criteria for identifying the appropriate moment to apply \ac{glai} during training, a method to construct \ac{glai} models with parameter counts comparable to the original \ac{mlp}, and a pruning criterion for the estimator.

\subsection{Structural and Quantitative Knowledge in MLPs}

We begin by formally defining what we mean by \emph{structural} and \emph{quantitative} knowledge in the context of \acp{mlp}. To set the stage, let us first fix the notation for the activation function. Hereafter, denote the \ac{relu} function by $\sigma: \mathbb{R}\rightarrow \mathbb{R}$, defined as $\sigma(z) = \max(0,z)$. For brevity, the same symbol $\sigma$ will also denote the component-wise extension $\mathbb{R}^n\rightarrow \mathbb{R}^n$, where $\sigma(z)_i = \sigma(z_i)$ for all $i \in \{1,\ldots,n\}$.

Although the definition of an \ac{mlp} is standard, we include it here briefly in order to unify notation and provide a consistent basis for the concepts introduced in this section.  

\begin{definition}
A \textbf{\acf{mlp}} with \ac{relu} activation and $L$ hidden layers is a mapping $f: \mathbb{R}^{n_0}\rightarrow \mathbb{R}^{n_{L+1}}$ that can be expressed as a composition $f=g_L\circ g_{L-1}\circ \ldots \circ g_1 \circ g_0$, where $g_0: \mathbb{R}^{n_0}\rightarrow \mathbb{R}^{n_1}$ is an affine mapping given by $g_0(x)=W_0 \cdot x+b_0$ with $W_0\in \mathbb{R}^{n_1 \times n_0}$ and $b_0\in \mathbb{R}^{n_1}$, and $g_l: \mathbb{R}^{n_l}\rightarrow \mathbb{R}^{n_{l+1}}$ is described as $g_l(x)=W_l \cdot \sigma(x)+b_l$, where $W_l\in \mathbb{R}^{n_{l+1} \times n_l}$ and $b_l\in \mathbb{R}^{n_{l+1}}$ for $l\in \{1, \ldots, L\}$. In other words, $f$ can be expressed as a composition
\begin{equation*}
\mathbb{R}^{n_0} \xrightarrow{W_0\cdot x+b_0} \mathbb{R}^{n_1} \xrightarrow{\sigma} \mathbb{R}^{n_1} \xrightarrow{W_1\cdot x+b_1} \mathbb{R}^{n_2} \rightarrow \ldots \rightarrow \mathbb{R}^{n_L} \xrightarrow{\sigma} \mathbb{R}^{n_L} \xrightarrow{W_L\cdot x+b_L} \mathbb{R}^{n_{L+1}},
\end{equation*}
alternating between affine transformations and \ac{relu} activations.
\end{definition}

For convenience, we denote by $f_l : \mathbb{R}^{n_0}\rightarrow \mathbb{R}^{n_{l+1}}$ the mapping $f_l = g_l \circ g_{l-1} \circ \ldots \circ g_0$ for $l\in \{0, \ldots, L\}$ that provides the intermediate values along the hidden layers in the neural network, in such a way that $f_0=g_0$ and $f_L = f$. 

\begin{remark} As is well-established, any affine transformation in the form of $x\mapsto W\cdot x+b$ can be regarded as a linear transformation by augmenting the dimensionality of both the input and output spaces with an additional unit, owing to the identity:
\begin{equation*}
    \begin{bmatrix}    W & b \\    0 & 1\end{bmatrix} \cdot\begin{bmatrix}    x \\    1\end{bmatrix} = \begin{bmatrix}    W\cdot x +b \\    1\end{bmatrix}.
\end{equation*}
Therefore, it is permissible, without loss of generality, to assume that $b_l = 0$ for all $l\in \{0,\ldots, L\}$. Consequently, this theoretical section exclusively considers \acp{mlp} devoid of bias parameters.  
\end{remark}

A pivotal definition in this mathematical framework is that of activation pattern, as it will form the basis for the subsequent construction.

\begin{definition}
For a fixed sample $x\in \mathbb{R}^{n_0}$, a neuron is said to be \textbf{active} if the value it outputs is positive, and \textbf{inactive} otherwise. The \textbf{activation pattern} of $x$ is defined as a list of $L$ vectors $(\text{act}_1(x), \ldots, \text{act}_L(x))$, where each $\text{act}_l(x)\in \{0,1\}^{n_l}$ contains $n_l$ binary values, determined based on whether the $n_l$ neurons of layer $l$ are active or inactive for the sample $x$. Formally, since the outputs of the $n_l$ neurons of layer $l$ are given by the vector $f_{l-1}(x)
\in \mathbb{R}^{n_l}$, it follows that
\begin{equation*}
\text{act}_l(x)= \sigma' \circ f_{l-1}(x),
\end{equation*}
where $\sigma'(z)_i = 1$ if $z_i>0$, and $\sigma'(z)_i = 0$ otherwise.
\end{definition}

The activation patterns of an \ac{mlp} capture its expressive capacity, and their evolution during training is central to the model’s ability to adapt to the data. A key phenomenon in \ac{relu}-based \acp{mlp} is that the network behaves linearly on the subset of inputs $x \in \mathbb{R}^{n_0}$ that share the same activation pattern, as formalized in the following result:

\begin{proposition}\label{prop:linear}
Let $A=(A_1, \ldots, A_L)$ denote a predefined activation pattern, and define the diagonal matrix $D_l = \textup{diag} (A_l)$ of size $n_l\times n_l$ where the diagonal elements are determined by the vector $A_l\in \{0,1\}^{n_l}$. Then, for every $x\in \mathbb{R}^{n_0}$ with activation pattern $A$, it holds that
\begin{equation*}
f(x)=W_L \cdot D_L \cdot W_{L-1} \cdot D_{L-1} \cdot \ldots \cdot W_1 \cdot D_1 \cdot W_0 \cdot x.
\end{equation*}
\end{proposition}
\begin{proof}
    We will prove by induction that $f_l (x) = W_l  \cdot  D_l \cdot  \ldots \cdot W_1 \cdot  D_1  \cdot W_0  \cdot x$ for all $l\in \{0,\ldots, L\}$. Since $f_L=f$, the result follows by setting $l=L$. For $l=0$, the result is trivial, as $f_0(x)=W_0\cdot x$ by definition. On the other hand, if $x$ has activation pattern $A$, then $\text{act}_l(x)=A_l$, implying that $\text{diag} (\text{act}_l(x)) =D_l$. Moreover, it naturally holds that $\sigma (z) = \text{diag} (\sigma'(z)) \cdot z$, and thus 
    \begin{align*}
        f_{l+1}(x) &= g_{l+1} \circ f_l (x) = W_{l+1} \cdot \sigma (f_l(x)) = W_{l+1} \cdot \text{diag} (\sigma'(f_l(x))) \cdot f_l (x) =\\
        & = W_{l+1}\cdot \text{diag} (\text{act}_{l+1}(x)) \cdot f_l(x) = W_{l+1} \cdot D_{l+1} \cdot f_l(x).
    \end{align*}
    Hence, the result follows.
\end{proof}
\begin{remark}
The fact established in the preceding proposition is clearly not a novelty, though it remains relatively underrepresented in the literature. Notably, \citet[Theorem~1]{sudjianto2020unwrappingblackboxdeep} already states this property in a related context. We include it here both for completeness and because we regard it as a fundamental yet underexploited perspective: despite its simplicity, the piecewise-linear nature of \ac{relu}-based networks rarely appears explicitly in modern treatments, even though it provides valuable insight for the developments that follow.
\end{remark}

Consequently, every \ac{relu}-based \ac{mlp} is a piecewise-linear function, where linearity holds within each region defined by a fixed activation pattern. In other words, activation patterns define regions of linearity of the network as a mapping. Once these stabilize, it is essentially a large but fixed piecewise-linear system.

Given a specific activation pattern, one can observe different paths across active neurons through which information flows in the \ac{mlp}. This phenomenon motivates the following definition:

\begin{definition}
    A \textbf{path} $\pi$ of the \ac{mlp} $f$ is a tuple $\pi=(\pi_0, \ldots, \pi_{L+1})$, where $\pi_l \in \{1, \ldots, n_l\}$ for all $l \in \{0, \ldots, L+1\}$. The index $\pi_0$ specifies the input coordinate where the path starts, the indices $\pi_1,\ldots,\pi_L$ indicate the positions of the hidden neurons traversed along the path, and $\pi_{L+1}$ denotes the output neuron where the path ends. 
\end{definition}

Consequently, one can visualize a path $\pi$ as a polygonal line across the neural network. Furthermore, each connection along a path, linking a neuron in layer $l$ to one in layer $l+1$, is associated with a weight, \textit{i.e.}, an element of the matrix $W_l$ that determines the contribution when moving from one layer to the next. This suggests introducing the following concept:
\begin{definition} 
    The \textbf{weight of a path} $\pi=(\pi_0, \ldots, \pi_{L+1})$ is the product of all the weights traversed along the path. Formally, if $w^l_{u,v}$ denotes the $(i,j)$ coordinate of the weight matrix $W_l$ associated with $f_l$, then the weight of $\pi$ is defined as
    \begin{equation*}
        \omega_\pi = w^0_{\pi_1,\pi_0}w^1_{\pi_2, \pi_1} \ldots w^L_{\pi_{L+1}, \pi_L}.
    \end{equation*}
\end{definition}

Next, the concepts of active and inactive path are presented for a specific sample $x\in \mathbb{R}^{n_0}$, along with a couple of functions related to this notion.

\begin{definition} Let $\pi=(\pi_0, \ldots, \pi_{L+1})$ be a path of the neural network $f$.
\begin{enumerate}
    \item Given a sample $x\in \mathbb{R}^{n_0}$, $\pi$ is said to be \textbf{active path for} $x$ if all hidden neurons through which $\pi$ passes are active. Formally, this occurs when $\text{act}_l (x)_{\pi_l}=1$ for all $l\in \{1, \ldots, L\}$.
    \item Denote the \textbf{indicator function} of $\pi$ as the function $\text{ind}_\pi:\mathbb{R}^{n_0}\rightarrow \{0,1\}$ defined as $\text{ind}_\pi(x) = 1 $ when $\pi$ is active for the sample $x$, and $\text{ind}_\pi(x)=0$ otherwise. 
    \item Define the \textbf{contribution function} of $\pi$ as the function $c_\pi:\mathbb{R}^{n_0}\rightarrow \mathbb{R}$ obtained as $c_\pi(x)=\text{ind}_\pi(x) \cdot x_{\pi_0}$, where $x_{\pi_0}$ is the coordinate of $x$ from which $\pi$ starts. In other terms, $c_\pi(x)$ returns the coordinate of $x$ from which path $\pi$ originates when the path is active, and returns $0$ otherwise.
\end{enumerate}
\end{definition}

The definition of the contribution function may initially appear arbitrary, yet its role becomes clear once we recognize that a path influences the \ac{mlp} solely through the coordinate from which it originates. This perspective is formalized in the following theorem, which shows that $f$ can be expressed as a linear combination of the contribution functions of all paths, each weighted by its corresponding parameters.

\begin{theorem}\label{thm:decomp}
Let $f$ be a \ac{relu}-based \ac{mlp}, let $c_1, c_2, \ldots, c_P$ denote the contribution functions of the $P$ paths of $f$ terminating at neuron $i$ in the last layer, with $i\in \{1, \ldots, n_{L+1}\}$, and let $\omega_1, \ldots, \omega_P$ represent their associated weights. Then,
\begin{equation*}
f(x)_i = \sum_{p=1}^P \omega_p c_p(x).
\end{equation*}
\end{theorem}

\begin{proof}
Let $x\in \mathbb{R}^{n_0}$ be a fixed sample. By Proposition \ref{prop:linear}, if one writes $D_l = \text{diag}(\act_l(x))$ for $l\in \{0, \ldots, L\}$, it then follows that
\begin{equation*}
f(x) = W_L \cdot D_L \cdot W_{L-1} \cdot D_{L-1} \cdot \ldots \cdot W_1 \cdot D_1 \cdot W_0 \cdot x.
\end{equation*}
Rewriting this product in terms of the components $w^l_{u,v}$ of the matrices $W_l$, for a fixed $i\in \{1, \ldots, n_{L+1}\}$, yields that 
\begin{align*}
        f(x)_i & = \sum_{i_L=1}^{n_L} \ldots \sum_{i_0=1}^{n_0} w^L_{i, i_L}  \act_k(x)_{i_L}  w^{L-1}_{i_L, i_{L-1}}  \act_{L-1}(x)_{i_{L-1}}  \ldots  w^1_{i_2, i_1}  \act_1(x)_{i_{1}}  w^0_{i_1, i_0} x_{i_0} \\
        &= \sum_{i_L=1}^{n_L} \ldots \sum_{i_0=1}^{n_0} w^0_{i_1, i_0} \ldots  w^{L-1}_{i_L, i_{L-1}} w^L_{i, i_{L}}  \act_1(x)_{i_{1}} \ldots \act_{L-1}(x)_{i_{L-1}}\act_L(x)_{i_L} x_{i_0}.
    \end{align*}
On the one hand, the product of weights corresponds to the weight of the path passing through the neurons in positions $i_0, \ldots, i_L , i$, which is the path $(i_0, i_1, \ldots, i_L, i)$. On the other hand, notice that the definition given for active path yields that the indicator function can be calculated as the product of the binary values
\begin{equation*}
    \text{ind}_{(i_0, i_1, \ldots, i_L, i)}(x) = \text{act}_1(x)_{i_1}\cdot \text{act}_2(x)_{i_2} \cdot \ldots \cdot \text{act}_L(x)_{i_L}.
\end{equation*}
Indeed, the product is $1$ if and only if each neuron of the path is active for $x$, and $0$ otherwise. Hence, the product of the activations with the input $x_{i_0}$ corresponds to the contribution function of the same path. It then follows that
\begin{equation*}
    f(x)_i = \sum_{i_L=1}^{n_L} \ldots \sum_{i_0=1}^{n_0} w_{(i_0, i_1, \ldots, i_L, i)}c_{(i_0, i_1, \ldots, i_L, i)}(x).
\end{equation*}
As this summation is carried out over all paths terminating at neuron $i$ in the last layer, a relabeling of the paths yields the expression
\begin{equation*}
f(x)_i = \sum_{p=1}^P w_p c_p(x),
\end{equation*}
establishing the theorem's statement.
\end{proof}
\begin{remark}
A related formulation to Theorem~\ref{thm:decomp} appears in \citet{meng2021mathcalgsgdoptimizingreluneural}, whose Equation~(1) also expresses the output of a \ac{relu}-based \ac{mlp} as a sum over active paths with effective weights. While their result provides an expression that is mathematically close to ours, its role in their work is limited to the analysis of optimization dynamics. In contrast, our contribution lies in leveraging this decomposition as the foundation of a new architectural paradigm which explicitly exploits the separation between structural and quantitative knowledge to accelerate training. Thus, although the algebraic resemblance may not be entirely novel, the theoretical perspective and its implications for model design introduced here are completely original.
\end{remark}

This demonstrates that every \ac{mlp} can be expressed as a linear combination of the contribution functions of its constituent paths, weighted by their associated path weights. Specifically, this formally illustrates that the knowledge held by an \ac{mlp} can be divided into two distinct types:
\begin{itemize}
\item The \textbf{structural knowledge}, entirely determined by the paths forming the model (specifically, through the contribution functions $c_p(x)$ of each path); and
\item The \textbf{quantitative knowledge}, determined by the weights $\omega_p$ associated to the paths of the \ac{mlp}.
\end{itemize}

This conceptual separation is not merely of theoretical interest. As will be shown later in Section \ref{subsec:path_distance}, the structural component of a trained \ac{dnn} stabilizes early during training and can be preserved without loss in validation scores. Consequently, once an \ac{mlp} has reached a sufficiently mature stage of structural knowledge, it can be reformulated according to Theorem~\ref{thm:decomp}, allowing the training to focus solely on the linear component associated with quantitative knowledge. The framework derived from this reformulation provides the foundation for the developments presented in the rest of this work.  

\subsection{The GLAI Framework}

Building on the previous definitions and results, we now formally introduce the \ac{glai} framework, presented here for the first time in the literature. This framework defines a novel paradigm in which models are decomposed into two complementary components, corresponding to the two forms of knowledge inherent to an \ac{mlp}: structural and quantitative.

The key idea is to abstract the notion of path in an \ac{mlp} by retaining only its contribution function and associated weight. As shown in Theorem~\ref{thm:decomp}, this information suffices to recover the full output of the network. Models within the \ac{glai} framework can thus be interpreted as linear with respect to the contribution functions, while these functions themselves encode the nonlinear structure of the data. This separation ensures fast training while preserving expressive power, since contribution functions capture meaningful nonlinearities.

\begin{definition}
A \textbf{\ac{glai} model} designed to infer features $y \in \mathbb{R}^m$ from samples $x \in \mathbb{R}^n$ is a mapping $\phi:\mathbb{R}^n \rightarrow \mathbb{R}^m$, where for each $i \in \{1, \ldots, m\}$, the output coordinate $y_i = \phi(x)_i$ is determined by
\begin{enumerate}[label=(\arabic*)]
\item contribution functions $c_p(x)$, for $p \in \{1, \ldots, P_i\}$, which return a fixed coordinate of $x$ within a defined piecewise linear region of the sample space, and 0 otherwise; and
\item associated weights $\omega_p \in \mathbb{R}$ for $p \in \{1, \ldots, P_i\}$,
\end{enumerate}
so that
$$\phi(x)_i = \sum_{p=1}^{P_i} \omega_p c_p(x).$$
\end{definition}

\begin{remark}
In the \ac{glai} framework, the concept of path is abstracted: instead of corresponding to a specific sequence of neurons in the network, a ``path'' is represented solely by a contribution function and its associated weight. This abstraction extends the notion of path beyond the strict architecture of \acp{mlp}. Moreover, the number of paths used to define each output coordinate need not be identical, unlike in conventional \acp{mlp}.
\end{remark}

According to Theorem~\ref{thm:decomp}, every \ac{relu}-based \ac{mlp} can be exactly expressed as a \ac{glai} model, which implies that the class of functions realized by \acp{mlp} is strictly contained within that of \ac{glai}. Since each output coordinate of a \ac{glai} model is built as a finite affine combination of contribution functions, and these are piecewise affine, the resulting mapping is always piecewise affine. Thus, the \ac{glai} framework preserves the desirable property of piecewise linearity while extending expressiveness beyond traditional \acp{mlp}.

A distinctive advantage of the \ac{glai} representation is the notion of \textbf{path virtualization}, which treats paths as independent entities. In standard \acp{mlp}, different paths may overlap through shared neurons, so a change in a single parameter $w^i_{u,v}$ affects all paths that traverse neuron $v$ in layer $i$ and neuron $u$ in layer $i+1$. By contrast, the \ac{glai} framework assigns independent weights to each path, effectively decoupling their influence. This virtualization enhances model flexibility and, as shown in Section~\ref{sec:experimental_results}, expands representational capacity compared to conventional \acp{mlp}. Furthermore, once the structural knowledge of the network has been fixed, path virtualization provides a mechanism to mitigate accuracy loss during retraining, a property confirmed empirically in our experiments.

In this setting, each coordinate $\phi(x)_i$ of a \ac{glai} model $f$ can be naturally described as a two-stage process. The first stage, the \textbf{path selector}, is defined by a mapping $S_i:\mathbb{R}^n \to \mathbb{R}^{P_i}$ that assigns to each input $x \in \mathbb{R}^n$ a vector $S_i(x) = (c_1(x), \ldots, c_{P_i}(x))$, where each component $c_p(x)$ represents the contribution of a path associated with the output coordinate $i$. The second stage, the \textbf{estimator}, is a linear mapping $L_i:\mathbb{R}^{P_i} \to \mathbb{R}$, given by $L_i(c_1, \ldots, c_{P_i}) = \omega_1c_1 + \cdots + \omega_{P_i}c_{P_i}$. Together, these two stages yield the decomposition $f(x)_i = L_i \circ S_i(x)$: first, the active paths are identified through $S_i(x)$, and then their contributions are aggregated by $L_i$ through a weighted linear combination.  

For a compact expression of the full output, let $S(x) = (S_1(x), \ldots, S_m(x))$ denote the concatenated vector of all path contributions, with $P_1+\ldots + P_m$ entries ordered by their corresponding output coordinates. If $\omega_p^i$ denotes the weight associated with path $p$ targeting coordinate $i$, $\tilde{\omega}^i$ denotes the row vector $[\omega_1^i, \, \omega_2^i, \, \ldots,\, \omega_{P_i}^i ]$ and $\Lambda$ denotes the matrix
\begin{equation*}
    \begin{bmatrix}
    \tilde{\omega}^1 &  &  & \\
     & \tilde{\omega}^2 &  & \\
     &  & \ddots & \\
     &  &  &\tilde{\omega}^m
\end{bmatrix}\in \mathbb{R}^{m \times (P_1+\ldots+ P_m)},
\end{equation*}
then $f(x) = \Lambda \cdot S(x)^T$. This formulation highlights the two-stage nature of \ac{glai} models: an input $x \in \mathbb{R}^n$ first propagates through the path selector stage, which determines the active paths and constructs $S(x)$. In the second stage, the estimator, a structured linear operator with parameter matrix $\Lambda$, combines these contributions to produce the output, with each path linked exclusively to a single output neuron (see Figure~\ref{fig:glai}).

\begin{figure}[h]
\centering
\includegraphics[]{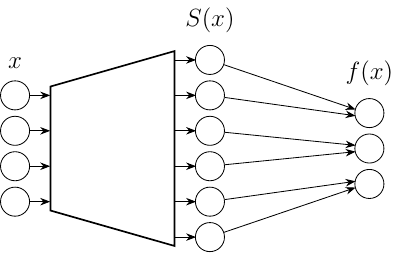}
\caption{Representation of a \ac{glai} model for samples $x\in \mathbb{R}^4$, target values $f(x)\in \mathbb{R}^3$. In this example representation, there are 6 paths in total, distributed in a ratio of 2 paths per output coordinate.}
\label{fig:glai}
\end{figure}

\subsection{GLAI in Practice}

In practical terms, applying \ac{glai} requires starting from an \ac{mlp} whose structural knowledge has reached a sufficiently stable stage. To establish a fair comparison and highlight the capabilities of \ac{glai}, we begin with a reference \ac{mlp} (hereafter the original \ac{mlp}), chosen with an architecture suitable for the target task and trained until full convergence.

For the \ac{glai} framework, however, we proceed differently. We start from a smaller \ac{mlp}, trained only for a reduced number of epochs enough to ensure stabilization of its structural knowledge but not full convergence. At this point, the \ac{glai} framework is applied: the structural knowledge of the reduced network is frozen, and the model is rewritten as a linear estimator defined over the path space. Concretely, the estimator is a linear model with one parameter per path, operating on $S(x)$, which encodes for each input $x$ the contribution functions of the corresponding paths. Although the number of paths grows exponentially with network depth, only a small fraction contributes effectively to the output. This redundancy enables aggressive pruning, drastically reducing the size of the estimator while preserving predictive accuracy.

The pruned estimator is then trained to convergence. The pruning ratio is selected so that the final estimator has a number of parameters equal to the difference between those of the original \ac{mlp} and the reduced \ac{mlp}, ensuring a fair comparison between \ac{mlp} and \ac{glai}. As a result, the full \ac{glai} framework unfolds in two phases: (i) a short training stage for a reduced \ac{mlp}, sufficient to stabilize structural knowledge, and (ii) the training of the pruned estimator obtained from rewriting this reduced \ac{mlp} within the \ac{glai} framework. Section~\ref{sec:experimental_results} confirms that this two-phase procedure consistently requires substantially less training time than the original \ac{mlp}, while achieving similar or even superior validation scores in most scenarios.

Figure~\ref{fig:glai_diagram} shows a diagram of how the \ac{glai} pipeline is carried out in practice. First, the reduced \ac{mlp} is trained and then stored as a fixed copy, which is used through forward passes to obtain the corresponding activations. Subsequently, the linear system derived from the estimator is trained by updating the weights associated with the paths, initialized from the already trained reduced model.

\begin{figure}[h]
    \centering
    \includegraphics[width=1\linewidth]{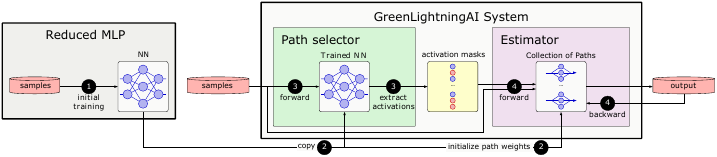}
    \caption{\ac{glai} training process. A reduced \ac{mlp} is first trained for a limited number of epochs to stabilize its structural knowledge. The network is then rewritten within the \ac{glai} framework by splitting it into a fixed path selector, which preserves the learned structure, and a trainable linear estimator defined over the selected path contributions. Finally, the estimator is pruned and trained to convergence while the path selector remains frozen.}
    \label{fig:glai_diagram}
\end{figure}

The next subsections present the implementation details on how the operations are performed as well as formalize the application criterion, the pruning strategy, and the parameter-count matching procedure used throughout our experiments.

\subsection{Implementation}
\label{sec:glai:implementation}

Listing~\ref{lst:glai:pytorch_reference} shows a simplified PyTorch implementation of the estimator. The computation first combines the input samples with the binary activations produced by the path selector, and then applies the estimator weights to obtain the output.

\begin{lstlisting}[
language=Python,
caption={Reference PyTorch implementation of the path-conditioned estimator},
label=lst:glai:pytorch_reference
]
# x_input:     [batch_size, n_inputs]
# weight_act:  [batch_size, n_weights]
# self.weight: [n_outputs, n_inputs, n_weights]

outer = torch.einsum("bi,bs->bis",
    (x_input, weight_act.float()))

output = torch.einsum("bis,ois->bo",
    (outer, self.weight))
\end{lstlisting}

Although this implementation is concise, it explicitly materializes the intermediate tensor containing the merged samples and path activations. To avoid this overhead, the optimized implementation uses custom CUDA kernels that directly accumulate the contribution of the active paths, as shown in Listing~\ref{lst:glai:forward_kernel}. The listing presents the computation as a sequential loop for clarity, while the actual CUDA kernel parallelizes the work across samples, outputs, and estimator components. As a result, the system only needs to store the weights and their gradients, together with the inputs and outputs of the DNN, in the same manner as any other DNN. A detailed comparison between the number of trainable parameters in GLAI and in an equivalent MLP is provided in Section~\ref{subsec:GLAI_pruning}.

\begin{lstlisting}[
language=C,
caption={Forward estimator},
label=lst:glai:forward_kernel
]
for(o = 0; o < n_outputs; o++){
    for(s = 0; s < n_samples; s++){
        y[s,o] = 0;

        for(i = 0; i < n_inputs; i++){
            sum_weights = 0;

            for(j = 0; j < n_weights; j++){
                sum_weights += weight_activations[s,j] * weights[i,j,o];
            }

            y[s,o] += sum_weights * x[s,i];
        }
    }
}
\end{lstlisting}

Since the custom CUDA operation is not visible to PyTorch's \texttt{autograd} system, the backward pass is implemented explicitly. Listing~\ref{lst:glai:backward_kernel} shows the corresponding gradient computation for the estimator weights.

\begin{lstlisting}[
language=C,
caption={Backward estimator},
label=lst:glai:backward_kernel
]
for(o = 0; o < n_outputs; o++){
    for(s = 0; s < n_samples; s++){
        dy_partial = dy[s,o];

        for(i = 0; i < n_inputs; i++){
            partial = x[s,i] * dy_partial;

            for(j = 0; j < n_weights; j++){
                dw[i,j,o] += weight_activations[s,j] * partial;
            }
        }
    }
}
\end{lstlisting}

The custom kernels therefore preserve the behavior of the reference implementation while avoiding the explicit construction of the intermediate path-conditioned tensor.

\subsection{Path Distance Definition}\label{subsec:theoretical_path_distance}

We next introduce a notion of distance between two paths of a \ac{glai} model. This definition will be a key ingredient in the remainder of this section, enabling the practical criteria and algorithms required to deploy \ac{glai}.

To define the path distance function for a \ac{glai} model $\phi:\mathbb{R}^{n_0}\to \mathbb{R}^{n_{L+1}}$, we fix a reference set $\Omega \subseteq \mathbb{R}^{n_0}$ on which $\phi$ is expected to achieve reliable scores. In practice, $\Omega$ can be chosen as the training or validation set, although the theoretical development remains valid for any arbitrary set $\Omega$, whether finite or infinite.  

\begin{definition}
Let $\Omega$ be a finite set, and consider two paths $\pi, \tilde{\pi}$ of $\phi$ originating from the same input coordinate $x_i$, with $i \in \{1,\ldots,n_0\}$. The distance between $\pi$ and $\tilde{\pi}$ with respect to $\Omega$ is defined as
\[
d_{\Omega}(\pi,\tilde{\pi}) = \frac{1}{|\Omega|}\sum \{|x_i|: x \in \Omega, \text{ind}_\pi(x)\neq \text{ind}_{\tilde{\pi}}(x)\} .\]
Intuitively, $d_{\Omega}(\pi,\tilde{\pi})$ measures the number of samples $x \in \Omega$ for which the two paths are not simultaneously active or inactive, weighted by the magnitude of the initial input coordinate $|x_i|$, and normalized by the cardinality $|\Omega|$.
\end{definition}

At a theoretical level, this definition can be extended to any set $\Omega$ equipped with a measure $m$. In that case, we define
\[
d_{\Omega}(\pi,\tilde{\pi}) = \frac{1}{m(\Omega)} \int_{\{x\in \Omega: \,\mathrm{ind}_\pi(x)\neq \mathrm{ind}_{\tilde{\pi}}(x)\}} |x_i| \, dx.
\]
Naturally, if $\Omega$ is finite and $m$ is the counting measure, this integral reduces to the discrete definition above. For this reason, the theoretical exposition will employ the general measure-based notation, while keeping in mind that in practice $\Omega$ will typically be a finite subset of the training samples of the network. From now on, we will simply write $d(\pi,\tilde{\pi})$ whenever the reference set $\Omega$ is clear from context.  

The next result reformulates the distance function in terms of the contribution functions:
\begin{proposition}
Let $c$ and $\tilde{c}$ be the contribution functions associated with paths $\pi$ and $\tilde{\pi}$, respectively. Then,
\[
d(\pi,\tilde{\pi}) = \frac{1}{m(\Omega)} \int_{\Omega} |c(x) - \tilde{c}(x)| \, dx.
\]
\end{proposition}

\begin{proof}
Observe that $|\mathrm{ind}_\pi(x) - \mathrm{ind}_{\tilde{\pi}}(x)| = 1$ if and only if the paths $\pi$ and $\tilde{\pi}$ are not simultaneously active or inactive for $x$. Since $c(x) = x_i \cdot \mathrm{ind}_\pi(x)$ and $\tilde{c}(x) = x_i \cdot \mathrm{ind}_{\tilde{\pi}}(x)$, it follows that
\[
\int_{\{x\in \Omega: \,\mathrm{ind}_\pi(x) \neq \mathrm{ind}_{\tilde{\pi}}(x)\}} |x_i| \, dx
= \int_{\Omega} |x_i| \cdot |\mathrm{ind}_\pi(x) - \mathrm{ind}_{\tilde{\pi}}(x)| \, dx
= \int_{\Omega} |c(x) - \tilde{c}(x)| \, dx,
\]
which completes the proof.
\end{proof}

Recall that the normalized $\ell_1$-norm of an integrable function $g$ over a set $\Omega$ is defined as
\[
\|g\|_1 = \frac{1}{m(\Omega)} \int_{\Omega} |g(x)| \, dx.
\]
This quantity corresponds to the average absolute value of $g$ over $\Omega$. With this notation, the formula from the previous proposition can be compactly expressed as
\[
d(\pi,\tilde{\pi}) = \|c - \tilde{c}\|_1.
\]
In other words, the distance between two paths can be interpreted simply as the $\ell_1$ distance between their respective contribution functions.

In these terms, the set of paths can itself be regarded as a normed space, by identifying each path $p$ originating from the coordinate $x_r$ with its contribution function $c$. Specifically, we define
\[
\|\pi\|_1 = \|c\|_1 = \frac{1}{m(\Omega)} \int_{\Omega} |c(x)| \, dx 
= \frac{1}{m(\Omega)} \int_{\{x\in \Omega : \,\mathrm{ind}_\pi(x) = 1\}} |x_i| \, dx.
\]
This norm will be particularly relevant when deciding which paths to prune from an estimator that is too large, since many paths in the network exhibit sufficiently small norms to be safely disregarded.

\subsection{GLAI Application Criterion Based on Path Distance}\label{subsec:path_distance}

One of the first practical questions that arises when applying the \ac{glai} framework is how many epochs are required before the structural knowledge of an \ac{mlp} becomes sufficiently mature to justify replacing it with an equivalent \ac{glai} model and retraining only its quantitative component. As previously discussed, \citet{duato} demonstrated that structural knowledge converges faster than the general knowledge of the \ac{mlp}. However, determining the precise point of convergence remains an open challenge.  

To address this, one can adapt the metric proposed in the cited work: activation patterns are computed for a set of samples at each epoch, and the distance between successive patterns for the same sample is averaged across the dataset. While this metric performs reasonably well in practice, the natural theoretical step after establishing the framework introduced here is to quantify differences in structural knowledge directly through paths, following an analogous procedure.  

Formally, let $\Omega$ be a subset of training samples. After each epoch of standard training of an \ac{mlp}, we compute the norms of the $P$ paths of the network. Suppose training proceeds for $T$ epochs, and denote by $c_1^t, \ldots, c_P^t$ the contribution functions at epoch $t \in \{1,2,\ldots,T\}$, which evolve as the network weights change (while the paths themselves remain fixed, their activations vary). We then define the metric
\[
m_t = \frac{1}{P} \sum_{p=1}^P d(c_p^t, c_p^{t+1}),
\]
which quantifies the change in structural knowledge of the \ac{mlp} across consecutive epochs.  

Figure~\ref{fig:convergence} illustrates the behavior of the proposed metric throughout training. The right $y$-axes report both the path distance $m_t$ and the distance between model weights, computed as the mean absolute difference of consecutive parameter values. The left $y$-axis displays the network’s validation loss, while the $x$-axis corresponds to training epochs. Training was carried out with a constant learning rate of $10^{-3}$ and weight decay of $10^{-2}$ using \ac{sgd}. Although more advanced optimizers or schedulers could have been employed, we deliberately opted for this simple setting to provide clean results, free from potential confounding effects due to dynamic hyperparameter updates. The comparison highlights the faster stabilization of path distances relative to model weights.

\begin{figure}[h]
\centering
\includegraphics[width=0.7\linewidth]{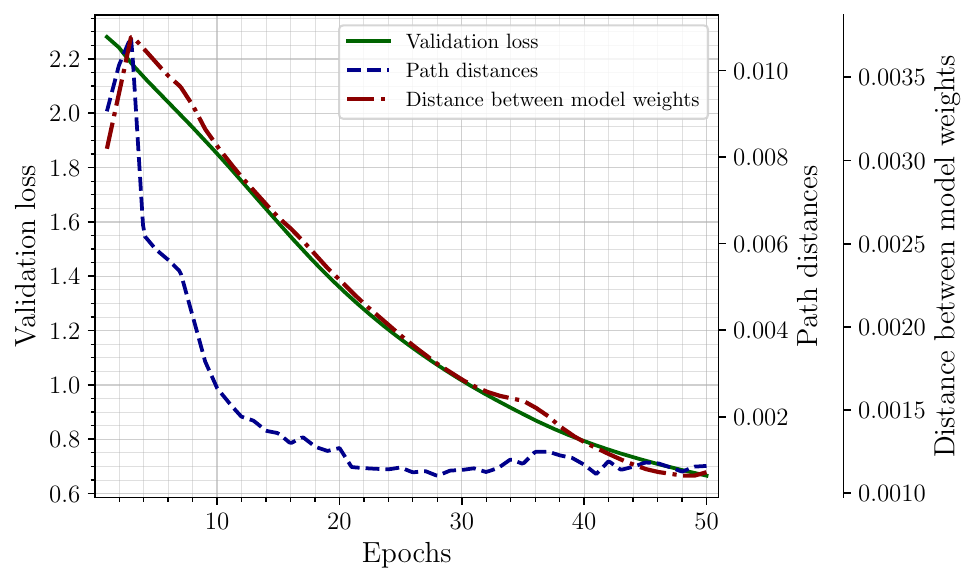}
\caption{Evolution of the path distance $m_t$ during training. The right $y$-axis corresponds to the relative error of $m_t$, while the left $y$-axis shows training and validation losses.}
\label{fig:convergence}
\end{figure}

These results confirm, within the framework of this novel path-based metric, the earlier observations of \citet{duato}: structural knowledge stabilizes significantly earlier than the full convergence of the \ac{mlp}, \emph{suggesting it can be frozen at this stage while continuing to train only the quantitative knowledge}.

\subsection{Estimator Pruning via Path Norms}\label{subsec:estimator_pruning}
Another practical issue that arises when applying the \ac{glai} framework concerns the size of the estimator. When the layers of an \ac{mlp} are compacted into a single layer, the number of paths grows exponentially with the depth of the network. Although this is not typically a severe limitation, since most \acp{mlp} used in practice rarely exceed two hidden layers, it becomes important to control this growth in order to ensure a fair comparison between a standard \ac{mlp} and its \ac{glai} counterpart. This is achieved by pruning the paths of a \ac{glai} model.  

In practice, one may start with an \ac{mlp} $\phi$ and consider an approximation $\tilde{\phi}$ with a reduced number of paths, under the expectation that this approximation remains sufficiently accurate so that the distance between $\phi$ and $\tilde{\phi}$ is small. To quantify the error between two \ac{glai} models $\phi, \tilde{\phi}:\mathbb{R}^{n_0}\to \mathbb{R}$ with $n_{L+1}=1$, over an arbitrary set $\Omega$, we use
\[
\|\phi - \tilde{\phi}\|_1 = \frac{1}{m(\Omega)} \int_{\Omega} |\phi(x) - \tilde{\phi}(x)| \, dx.
\]
If $\Omega$ is finite (for instance, the training set in a practical scenario), $\Omega = \{x_1,\ldots, x_m\}$, this reduces to
\[
\|\phi - \tilde{\phi}\|_1 = \frac{1}{m} \sum_{k=1}^m |\phi(x_k) - \tilde{\phi}(x_k)|,
\]
which corresponds to the average error across the samples $x_k$.  

More generally, if $\phi, \tilde{\phi} : \mathbb{R}^{n_0}\to \mathbb{R}^{n_{L+1}}$, the distance can be computed by summing the distances between their coordinate functions $\phi_i, \tilde{\phi}_i : \mathbb{R}^{n_0}\to \mathbb{R}$ as
\[
\|\phi - \tilde{\phi}\|_1 = \sum_{i=1}^{n_{L+1}} \|\phi_i - \tilde{\phi}_i\|_1.
\]

The next result provides a practical criterion for compressing a \ac{glai} model by reducing the number of active paths. As in the preceding sections, denote by $c_{1,i}, \ldots, c_{P_i,i}$ the contribution functions of the $P_i$ paths of $\phi$ associated to the neuron $i$ of the last layer, with $i \in \{1,\ldots, n_{L+1}\}$. Likewise, let $w_{1,i}, \ldots, w_{P_i,i}$ be the associated weights, so that $$\phi(x)_i = \sum_{p=1}^{P_i} w_{p,i} \, c_{p,i}(x)$$ for each $i\in \{1, \ldots, n_{L+1}\}$.
\begin{theorem}
Let $E \subseteq \{(p,i)\in \mathbb{N}^2 : 1 \leq p \leq P_i,\ 1 \leq i \leq n_{L+1}\}$ be the set of indices corresponding to the paths to be removed, where $(p,i) \in E$ if and only if the $p$-th path ending at neuron $i$ is eliminated. Then the pruned network $\tilde{\phi}$ obtained by removing the paths in $E$ satisfies that
\[
\|\phi - \tilde{\phi}\|_1 \;\leq\; \sum_{(p,i)\in E} |w_{p,i}| \cdot \|c_{p,i}\|_1.
\]
\end{theorem}

\begin{proof}
By the definition of $E$, let $E_i = \{p : (p,i)\in E\}$ be the set of indices of the paths removed that end at neuron $i$, with $i \in \{1,\ldots,n_{L+1}\}$. Then
\[
\phi(x)_i = \tilde{\phi}(x)_i + \sum_{p\in E_i} w_{p,i} \, c_{p,i}(x).
\]
Hence,
\[
\|\phi_i - \tilde{\phi}_i\|_1 = \Big\| \sum_{p\in E_i} w_{p,i}\, c_{p,i}(x) \Big\|_1 \;\leq\; \sum_{p\in E_i} |w_{p,i}| \cdot \|c_{p,i}\|_1.
\]
Summing over all output coordinates gives
\[
\|\phi - \tilde{\phi}\|_1 = \sum_{i=1}^{n_{L+1}} \|\phi_i - \tilde{\phi}_i\|_1 \;\leq\; \sum_{i=1}^{n_{L+1}} \sum_{p\in E_i} |w_{p,i}| \cdot \|c_{p,i}\|_1
= \sum_{(p,i)\in E} |w_{p,i}| \cdot \|c_{p,i}\|_1 
\]
as claimed.
\end{proof}

Consequently, the set of paths can be pruned according to the values of the product between each path weight absolute value and its norm, namely the term $|w_{p,i}| \cdot \|c_{p,i}\|_1$ for the $p$-th path ending at neuron $i$. The error introduced by such pruning is in fact controlled, as it is bounded by the sum of these products over the discarded paths. To achieve a pruning factor $0<\mu<1$, meaning that only $100\cdot \mu \,\%$ of the paths are retained, one can compute the values $|w_{p,i}| \cdot \|c_{p,i}\|_1$ for all paths and select the top fraction corresponding to the desired quantile.

\section{Experimental Results}\label{sec:experimental_results}

The primary objective of this experimental section is to assess the practical value of the proposed \ac{glai} framework, building upon the theoretical foundations introduced earlier. Rather than focusing solely on raw speed, our goal is to show that \ac{glai} serves as an alternative to conventional fully connected heads in scenarios where such components are indispensable. In particular, we consider the widely adopted training paradigm where a lightweight head is optimized on top of a frozen backbone, a setting that naturally highlights the efficiency and stability of the approach in practical scenarios.  

To provide a broad evaluation, we design three families of experiments, each reflecting a different methodological context in which \acp{mlp} play a central role:  
\begin{enumerate}[label=(\Alph*)]
    \item \textbf{Fixed embedding classification}: specializing pretrained models on downstream tasks beyond their original training domain.  
    \item \textbf{Self-supervision}: improving representation quality when abundant unlabeled data are available, using contrastive or predictor-style objectives.  
    \item \textbf{Few-shot learning}: adapting to entirely new tasks from only a handful of labeled examples. 
\end{enumerate}

These scenarios are of particular importance across diverse application domains, such as industrial inspection (A), autonomous driving (B), and medical imaging (C), among many others. Together, they demonstrate how \ac{glai} can act as a drop-in replacement for \acp{mlp} while maintaining performance and reducing optimization burden.

Further technical details are provided in Subsection~\ref{subsec:evaluation_details}. Results are summarized in Subsection~\ref{subsec:results}, and the parameter-budget matching procedure used to ensure a fair comparison is analyzed in Subsection~\ref{subsec:GLAI_pruning}.

\subsection{Evaluation Details}\label{subsec:evaluation_details}

This subsection provides additional technical details regarding the experimental setup used throughout Section~\ref{sec:experimental_results}. Within each family, we explore multiple configurations of backbones and datasets, systematically replacing conventional \ac{mlp} heads with their \ac{glai}-based counterparts. Concretely, Family A includes experiments (A1) DeiT-S/16 \citep{vit} on Oxford-IIIT Pets \citep{oxford-IIIT-pets} and (A2) RoBERTa-base \citep{roberta} on DBPedia-14 \citep{dbpedia}. Family B covers (B1) EfficientNet-B0 \citep{efficientnet} on STL-10 unsupervised split \citep{stl-10} and (B2) GPT-2 small \citep{gpt-2} on WikiText-2 without labels \citep{wikitext}. Finally, Family C contains (C1) MobileNetV3-S \citep{mobilenet} on Omniglot \citep{omniglot}, and (C2) XLNet-base \citep{xlnet} on AGNews \citep{dbpedia}.  

All the code used to conduct these experiments is available in the GitHub repository at \url{https://github.com/jmiravet/GLAI}.  

The architectures in each experiment were chosen to represent different scales of the replaced models. We adopt the notation $(n_0, n_1, \ldots, n_L, n_{L+1})$, where $n_0$ and $n_{L+1}$ denote the number of input and output units respectively, and $n_1, \ldots, n_L$ the number of neurons in each of the $L$ intermediate layers. The selected architectures are summarized in Table~\ref{table:architectures}, where for each experiment (A1–C2) we list the corresponding tuple and the total parameter count (in millions).  

\begin{table}[h]
\centering
\caption{Architectures evaluated in the experiments.}
\label{table:architectures}
\small
\centering
\resizebox{\linewidth}{!}{
\begin{tabular}{c l l c c}
\toprule
\textbf{Exp.} & \textbf{Backbone} & \textbf{Dataset} & \textbf{Arc. parameters} & \textbf{\#Params (K)} \\
\midrule
A1 & DeiT-S/16         & Oxford Pets & (384, 256, 37)   & 108 \\
A2 & RoBERTa-base      & DBPedia     & (768, 128, 4)    & 100 \\
B1 & EfficientNet-B0   & STL-10      & (1280, 640, 128) & 901 \\
B2 & GPT-2             & WikiText-2  & (768, 768, 128)  & 689 \\
C1 & MobileNetV3-Small & Omniglot    & (576, 256, 5)    & 149 \\
C2 & XLNet-base        & AGNews      & (768, 512, 4)    & 396 \\
\bottomrule
\end{tabular}
}
\end{table}

The training procedure for \ac{glai} consists of two phases. First, a conventional \ac{mlp} is trained until convergence. Subsequently, an additional training stage is carried out for a reduced number of epochs to achieve structural convergence (see \ref{subsec:path_distance}). This stage relies on a reduced-size \ac{mlp} with reduction factor $\rho$ (see \ref{subsec:GLAI_pruning}), together with the corresponding $\mu$ associated with each $\rho$. The purpose of this step is to compress the estimator while preserving the alignment between structural and quantitative knowledge.  

Regarding training setups, each family of experiments follows a dedicated protocol aligned with its intended objective. The \textbf{A-family} of experiments employs a standard supervised learning setup, where frozen backbones provide embeddings and trainable heads are optimized with cross-entropy loss for classification. This design mirrors classical fine-tuning pipelines commonly adopted in practice. The \textbf{B-family} focuses on unsupervised learning, where models act as projection heads trained with a contrastive loss. In this setting, dropout is applied only to the trainable projection layers (excluding the frozen backbone), in order to generate multiple stochastic views of the same input for the computation of the InfoNCE loss. The dropout rate is set to $0.2$ in \textbf{B1} and $0.4$ in \textbf{B2}, while the temperature parameter remains fixed at $0.07$ in both cases. No data augmentation is employed, as the goal is to isolate the effect of the projection head and measure only the cost of its forward and backward passes, avoiding additional computational overhead from repeated inferences through the backbone. Finally, the \textbf{C-family} targets few-shot learning. In \textbf{C1}, we adopt a 5-way 4-shot 6-query configuration, where each episode comprises 20 support images (4 per class) and 30 query images (6 per class). In \textbf{C2}, the setup extends to relation extraction with a 5-way 5-shot 10-query configuration, yielding 25 support sentences and 50 query sentences per episode. Both C-family experiments follow the meta-learning paradigm: support examples are used for adaptation within each episode, while query examples evaluate generalization to unseen samples of the same classes.  


All conventional \ac{mlp} baselines, including the reduced \ac{mlp}, are optimized with \ac{sgd} under the conditions reported in Table~\ref{table:training}. For the \ac{glai} estimator, we use Adam. This choice is motivated by the fact that, after the structural component has been fixed, \ac{glai} no longer optimizes the same parameter space as a standard \ac{mlp}: training is performed over path weights in a linear estimator defined on frozen path contributions. This induces a different optimization geometry, with gradients that can be more anisotropic and a loss landscape that is often flatter than in the original \ac{mlp} parameterization. In this setting, using the same optimizer is not necessarily the fairest comparison, since \ac{sgd} and adaptive methods interact differently with each parameterization. We therefore keep \ac{sgd} for the conventional \acp{mlp}, where it provides a simple and stable baseline with minimal optimizer-induced effects, and use Adam for the \ac{glai} estimator, where coordinate-wise adaptive updates provide more stable optimization. Importantly, the reported wall-clock times include the full cost of the optimizer used in each case, so the measured speedups correspond to the actual training pipelines being compared. A deeper study of optimizer behavior and loss-geometry in the \ac{glai} parameter space is left for future work. Dropout was employed exclusively in the B-family experiments. The remaining hyperparameters are summarized in Table~\ref{table:training}, including learning rate (LR) and batch size (BS), which are identical for both \acp{mlp} and \ac{glai}. The table also reports the weight decay (WD) values, which differ between \acp{mlp} and their \ac{glai} counterparts. This difference is intentional: WD was selected independently for each architecture so as to use the regularization setting that most benefits each model. In particular, after the replacement procedure, the \ac{glai} estimator has a simpler, effectively single-layer structure, which makes it more prone to overfitting when weak regularization is used. For this reason, \ac{glai} requires a stronger WD than the baseline \ac{mlp} in order to improve generalization and provide a fair comparison between well-tuned versions of both approaches. Convergence is determined by early stopping with identical parameter settings for each \ac{mlp} and its corresponding \ac{glai}, ensuring a fair comparison. Validation accuracy is monitored for the A- and C-families, whereas validation loss is monitored for the B-family. Patience and minimum delta parameters are set according to the experiment: patience of 5 and min delta of 0.1 for A1 and A2; patience of 5 and min delta of $10^{-5}$ for B1; patience of 5 and min delta of 0.01 for B2; and patience of 30 with min delta of 1 for the C-family.

For \ac{glai} training, the values of WD and $\rho$ (the reduction factor) are reported, along with the number of epochs dedicated to the reduced \ac{mlp} stage (Red. Epochs in Table~\ref{table:training}), which is chosen to guarantee structural convergence. Finally, the total number of epochs required for \ac{glai} training is reported under Epochs to conv., representing the sum of the reduced \ac{mlp} training epochs and the subsequent epochs until full convergence of the estimator.

\begin{table}[h]
\centering
\caption{Training configurations and results. Experiments are identified by code only. Epochs to convergence are reported as mean $\pm$ std across seeds.}
\label{table:training}
\resizebox{\linewidth}{!}{
\begin{tabular}{@{\extracolsep{\fill}} c c c c c c c c c}
\toprule
\multirow{2}{*}{Exp.} & 
\multirow{2}{*}{LR} & 
\multirow{2}{*}{BS} & 
\multicolumn{2}{c}{MLP Training} & 
\multicolumn{4}{c}{GLAI Training} \\
\cmidrule(lr){4-5} \cmidrule(lr){6-9}
 & & & WD & Epochs to conv. & WD & $\rho$ & Red. epochs & Epochs to conv. \\
\midrule
A1 & 0.001  & 16 & 0.001 & $60.00 \pm 0.00$   & 0.1 & 0.5 & 20 & $14.33 \pm 3.21$ \\
A2 & 0.001  & 16 & 0.001 & $37.33 \pm 3.06$   & 0.1 & 0.5 & 20 & $6.67 \pm 0.58$ \\
B1 & 0.001  & 16 & 0.001 & $31.33 \pm 2.08$   & 0.1 & 0.5 & 60 & $8.33 \pm 3.51$ \\
B2 & 0.0001 & 16 & 0.01  & $9.33 \pm 1.53$    & 0.1 & 0.5 & 2  & $2.00 \pm 0.00$ \\
C1 & 0.001  & 16 & 0.001 & $154.00 \pm 110.69$ & 0.1 & 0.5 & 30 & $43.67 \pm 6.81$ \\
C2 & 0.001  & 32 & 0.01  & $78.00 \pm 38.11$  & 0.1 & 0.5 & 30 & $43.00 \pm 7.81$ \\
\bottomrule
\end{tabular}}
\end{table}

Determining when the structural knowledge of an MLP has stabilized is a challenging problem, both theoretically and in practice, as it requires understanding the dynamics of activation patterns and designing reliable proxies to measure them. In this work, we adopt a simple and reproducible heuristic. For each experiment, we first estimate the number of epochs required for full convergence of the corresponding MLP, and then train the reduced MLP for a fixed fraction of this budget: 20\% for Family A, 15\% for Family B, and 10\% for Family C. These values reflect the differences across training regimes and were found to provide a reasonable trade-off between structural maturity and efficiency. While this rule does not explicitly measure structural convergence, it offers a consistent practical criterion across all experiments.

All experiments were implemented in PyTorch and executed on a compute node running Ubuntu 18.04.5 LTS (Linux kernel 4.15.0). The node is equipped with two AMD EPYC 7282 processors (32 physical cores, 64 threads in total; 2 NUMA nodes; base frequency 1.5 GHz, maximum frequency 2.8 GHz) and ten NVIDIA A100 GPUs, each with 80 GB of memory. For all reported experiments, only a single A100 GPU was allocated and used. Jobs were submitted through Slurm with a memory request of 64 GB, although actual usage remained below this threshold.

\subsection{Ensuring a Fair Comparison Between an MLP and Its GLAI Counterpart}\label{subsec:GLAI_pruning}

Once the pruning procedure for reducing the number of paths in a \ac{glai} model has been established, it is necessary to address how to guarantee a fair comparison between a conventional \ac{mlp} and its associated \ac{glai} model. This issue arises because the \ac{glai} formulation requires the underlying \ac{mlp} to perform the forward pass in order to compute the path activations, which are then weighted by the estimator.  

Our investigation has shown that a straightforward yet effective strategy is to replace the original \ac{mlp} with a smaller one, reduced by a fixed factor. Such a reduced \ac{mlp} preserves the expressivity of activation patterns, can be trained more efficiently, and yields a \ac{glai} model that outperforms the original \ac{mlp} in training time while achieving comparable accuracy. 

Formally, consider an original \ac{mlp} with layer dimensions given by the tuple $(n_0, n_1, \ldots, n_L, n_{L+1})$, where $n_0$ and $n_{L+1}$ denote the input and output dimensions, respectively, and $L \geq 1$ is the number of hidden layers. We propose reducing each hidden layer size uniformly by a factor $0<\rho<1$, resulting in a reduced \ac{mlp} with dimensions $(n_0, \rho\cdot n_1, \ldots, \rho\cdot  n_L, n_{L+1})$. To prevent a bottleneck\footnote{If $n_L = n_{L+1}$, the proposed method no longer applies and must be reconsidered. Since the experimental setups in this work do not involve architectures of this shape, we omit further discussion here.} at the final hidden layer, we require $\rho \cdot n_L\geq n_{L+1}$.

The comparison proceeds as follows. The original \ac{mlp} is trained to convergence. Its \ac{glai} counterpart is obtained from the reduced \ac{mlp}, which is trained for only a fraction of the epochs required by the original, a sufficient amount to ensure the convergence of structural knowledge as detailed in Subsection~\ref{subsec:path_distance}. At this point, the equivalent \ac{glai} model is constructed from the reduced \ac{mlp} and subsequently pruned by a factor $\mu$ to match its parameter count with that of the original \ac{mlp}. Specifically, pruning ensures that the sum of the parameters of the frozen reduced \ac{mlp} (used only to compute activations) and the parameters of the estimator equals the number of parameters in the original model.  

The pruning factor $\mu$ is determined analytically. Let the number of parameters of the original network be
\[
\text{O} \;=\; \sum_{l=0}^{L} (n_l + 1)\, n_{l+1},
\]
the number of parameters of the reduced network be
\[
\text{R} \;=\; 
\rho \cdot \Bigl( n_0 n_1 + n_L n_{L+1} + \sum_{l=0}^{L} n_{l+1} \Bigr) 
+ \rho^2 \cdot \sum_{l=1}^{L-1} n_l n_{l+1} \;+\; n_{L+1},
\]
and the number of parameters of the estimator obtained from the reduced \ac{mlp} be
\[
\text{E} \;=\; 
\rho^d \cdot \prod_{l=0}^{L+1} n_l 
+ \sum_{k=1}^{L+1} \rho^{\,L+1-k}\cdot \prod_{l=k}^{L+1} n_l.
\]
Then, the pruning factor is given by $\mu = (\text{O} - \text{R})/\text{E}$, which guarantees that the comparison between the original \ac{mlp} and its \ac{glai} counterpart is balanced.

\begin{example}
Consider experiment \textbf{B1}, where the original \ac{mlp} has architecture
$(1280,640,128)$ and $\text{O}=901\text{K}$ parameters. Using $\rho=0.5$, the
reduced \ac{mlp} has $\text{R}=225\text{K}$ parameters, while the complete
estimator contains $\text{E}=13\text{M}$ parameters before pruning.

To match the memory footprint of the original \ac{mlp}, the estimator must keep
$\text{O}-\text{R}=901\text{K}-225\text{K}=676\text{K}$ parameters. Therefore,
the required pruning factor is $\mu=(\text{O}-\text{R})/\text{E}
=676\text{K}/13\text{M}
\approx 5.2\%.$ Thus, the final \ac{glai} model combines the reduced \ac{mlp}
($225\text{K}$ parameters) with the pruned estimator ($676\text{K}$ parameters),
matching the $901\text{K}$ parameters of the original model.
\end{example}
\begin{example}
Consider a simple theoretical setting to illustrate a potential limitation of the method when applied to very deep MLPs. Although such architectures are uncommon in practice, analyzing this regime helps clarify the behavior in extreme cases.

Let us assume an MLP with depth $d$ (i.e., $d$ hidden layers) and a constant width of $n$ neurons per layer. The total number of parameters in the original MLP is then $O = d n^2$. After applying a reduction factor $\rho$, the reduced model contains $R = d \rho n^2$ parameters, while the resulting estimator has size $E = \rho^{d+1} n^{d+1}$.

The key issue emerges from the exponential dependence of $E$ on the depth $d$. As $d$ increases, the estimator can grow rapidly, potentially leading to practical inefficiencies. In this setting, the corresponding pruning factor is given by
\[
\mu = \frac{d(1 - \rho^2)}{\rho^{d+1} n^{d-1}}.
\]

To provide a concrete illustration, consider $n = 128$, $d = 4$, and $\rho = 0.5$. In this case, we obtain $\mu \approx 4 \times 10^{-5}$, an extremely small value that may negatively affect the stability or effectiveness of the model.

While the estimator naturally enables aggressive pruning, this example highlights that, in certain edge cases, the required pruning levels may become excessively restrictive. This reinforces the importance of considering architectural depth when applying the method.
\end{example}

\begin{remark}
If $\rho$ is too small, it may occur that $\mu > 1$. This means that the estimator contains fewer parameters than required to match the parameter count of the original \ac{mlp}. In this case, one may either set $\mu=1$ (i.e., avoid pruning and obtain a \ac{glai} model with fewer parameters than the original) or select a larger value of $\rho$.  
\end{remark}

\begin{figure}[tb]
    \centering
    \begin{subfigure}{0.49\linewidth}
        \centering
        \caption{Experiment A1}
        \includegraphics[width=\linewidth]{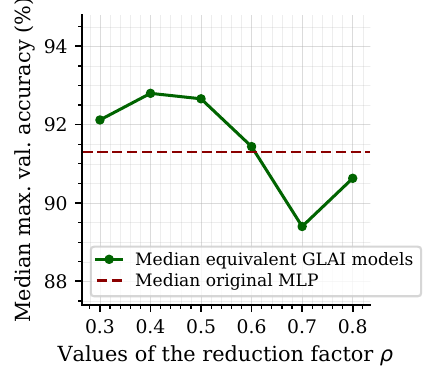}
        \label{fig:p_values_A1}
    \end{subfigure}
    \hfill
    \begin{subfigure}{0.49\linewidth}
        \centering
        \caption{Experiment C1}
        \includegraphics[width=\linewidth]{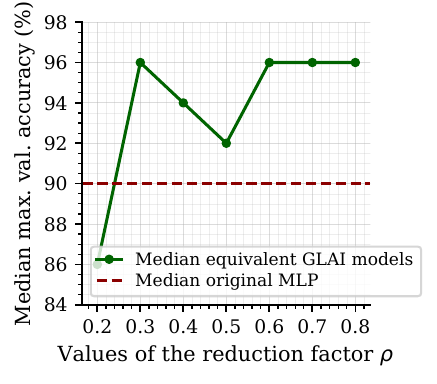}
        \label{fig:p_values_C1}
    \end{subfigure}
    \vspace{-0.5cm}
    \caption{Maximum validation accuracy obtained by the original \ac{mlp} and its \ac{glai} counterparts for different reduction factors $\rho$.}
    \label{fig:p_values}
\end{figure}

The balance between the two compression factors, $\rho$ and $\mu$, is delicate. While more sophisticated techniques for reducing the size of the \ac{mlp} could be employed, the method described here has proven sufficient for the scale of the experiments considered. Since $\mu$ is determined as a function of $\rho$, the only remaining choice is the value of $\rho$. Figure~\ref{fig:p_values} reports an ablation study on the effect of the reduction factor $\rho$ in two representative settings: experiment A1 (Oxford-IIIT Pets + DeiT-S/16) and experiment C1 (Omniglot + MobileNetV3-S). The $x$-axis reports the value of $\rho$, while the $y$-axis shows the median validation accuracy obtained by the original \ac{mlp} models and their corresponding \ac{glai} counterparts.

The results show that the effect of $\rho$ is not governed by a simple monotonic or low-order trend, reflecting the complexity of the transformations involved in replacing the original \acp{mlp} with \ac{glai} modules. Nevertheless, the intermediate value $\rho=0.5$ performs consistently well in both ablation settings. In particular, it remains robustly above the original \ac{mlp} baseline and provides a stable choice across architectures and datasets. Based on this observation, we use $\rho=0.5$ in the main experiments, as it represents a reasonable intermediate value that offers a consistent trade-off between reduction and validation performance.

\subsection{Results}\label{subsec:results}

The experimental results are summarized in Table~\ref{tab:experiments}\footnotetext{Datasets in Family B are considered in their unlabeled form to match the unsupervised learning setting.}. For each configuration (A1--C2), the table reports four alternatives: the original \ac{mlp}, a reduced \ac{mlp} using the same reduction factor adopted by \ac{glai} ($\rho=0.5$), the corresponding \ac{glai} model, and a linear head. All models are trained under the same Early Stopping criterion, so the reported number of epochs corresponds to training until convergence according to that criterion. Results are computed over three random seeds and reported as mean $\pm$ standard deviation. 

\begin{table}[ht]
\centering
\caption{Results across all experiment families (A: Fixed Embedding Classification; B: Self-Supervision\protect\footnotemark; C: Few-Shot Learning). Acronyms: \emph{Exp.} denotes experiment; \emph{Epochs} indicates the number of training epochs until early stopping, reported as mean $\pm$ std across seeds; \emph{Speedup} is the ratio of the early-stopped elapsed training time of the original \ac{mlp} to that of each counterpart; and \emph{BVS} stands for Best Validation Score (\% for accuracy, dimensionless values for loss).}
\label{tab:experiments}
\resizebox{\linewidth}{!}{
\begin{tabular}{@{\extracolsep{\fill}} c c c c c c}
\addlinespace[2pt]\toprule
Exp. & Backbone / Dataset & Head & Epochs & Speedup & BVS \\
\midrule\midrule

\multirow{4}{*}{A1} 
& \multirow{4}{*}{\shortstack{DeiT-S/16\\ \footnotesize Oxford-IIIT Pets}} 
& MLP & $60.00 \pm 0.00$ & $1\times$ & $\text{acc: }88.05 \pm 0.28\,\%$ \\
& & MLP reduced & $60.00 \pm 0.00$ & $(1.01 \pm 0.02)\times$ & $\text{acc: }86.60 \pm 1.52\,\%$ \\
& & GLAI & $14.33 \pm 3.21$ & $(2.54 \pm 0.52)\times$ & $\text{acc: }90.75 \pm 0.28\,\%$ \\
& & Linear & $53.00 \pm 7.55$ & $(1.32 \pm 0.13)\times$ & $\text{acc: }90.57 \pm 0.21\,\%$ \\

\addlinespace[2pt]\midrule

\multirow{4}{*}{A2} 
& \multirow{4}{*}{\shortstack{RoBERTa-base\\ \footnotesize DBPedia-14}} 
& MLP & $37.33 \pm 3.06$ & $1\times$ & $\text{acc: }97.70 \pm 0.01\,\%$ \\
& & MLP reduced & $37.00 \pm 1.00$ & $(1.02 \pm 0.08)\times$ & $\text{acc: }97.68 \pm 0.03\,\%$ \\
& & GLAI & $6.67 \pm 0.58$ & $(3.50 \pm 0.46)\times$ & $\text{acc: }97.89 \pm 0.16\,\%$ \\
& & Linear & $24.33 \pm 13.32$ & $(2.68 \pm 2.16)\times$ & $\text{acc: }95.44 \pm 0.04\,\%$ \\

\midrule\midrule

\multirow{4}{*}{B1} 
& \multirow{4}{*}{\shortstack{EfficientNet-B0\\ \footnotesize STL-10\tnote{1}}} 
& MLP & $31.33 \pm 2.08$ & $1\times$ & $\text{loss: }(1.63 \pm 0.55)\cdot 10^{-4}$ \\
& & MLP reduced & $23.33 \pm 2.89$ & $(0.91 \pm 0.14)\times$ & $\text{loss: }(1.87 \pm 0.25)\cdot 10^{-4}$ \\
& & GLAI & $8.33 \pm 3.51$ & $(1.05 \pm 0.40)\times$ & $\text{loss: }(1.70 \pm 0.53)\cdot 10^{-4}$ \\
& & Linear & $60.00 \pm 0.00$ & $(0.81 \pm 0.05)\times$ & $\text{loss: }(3.00 \pm 1.30)\cdot 10^{-4}$ \\

\addlinespace[2pt]\midrule

\multirow{4}{*}{B2} 
& \multirow{4}{*}{\shortstack{GPT-2 small\\ \footnotesize WikiText-2\tnote{1}}} 
& MLP & $9.33 \pm 1.53$ & $1\times$ & $\text{loss: }0.209 \pm 0.010$ \\
& & MLP reduced & $9.67 \pm 0.58$ & $(0.97 \pm 0.12)\times$ & $\text{loss: }0.220 \pm 0.013$ \\
& & GLAI & $2.00 \pm 0.00$ & $(1.56 \pm 0.25)\times$ & $\text{loss: }0.159 \pm 0.013$ \\
& & Linear & $6.33 \pm 1.15$ & $(2.30 \pm 0.26)\times$ & $\text{loss: }0.181 \pm 0.002$ \\

\midrule\midrule

\multirow{4}{*}{C1} 
& \multirow{4}{*}{\shortstack{MobileNetV3-S\\ \footnotesize Omniglot}} 
& MLP & $154.00 \pm 110.69$ & $1\times$ & $\text{acc: }90.00 \pm 8.00\,\%$ \\
& & MLP reduced & $188.33 \pm 30.66$ & $(0.91 \pm 0.64)\times$ & $\text{acc: }90.00 \pm 10.00\,\%$ \\
& & GLAI & $43.67 \pm 6.81$ & $(1.85 \pm 1.24)\times$ & $\text{acc: }96.67 \pm 3.06\,\%$ \\
& & Linear & $198.33 \pm 25.42$ & $(1.07 \pm 0.89)\times$ & $\text{acc: }95.33 \pm 6.43\,\%$ \\

\addlinespace[2pt]\midrule

\multirow{4}{*}{C2} 
& \multirow{4}{*}{\shortstack{XLNet-base\\ \footnotesize AGNews}} 
& MLP & $78.00 \pm 38.11$ & $1\times$ & $\text{acc: }52.50 \pm 5.00\,\%$ \\
& & MLP reduced & $66.33 \pm 41.48$ & $(2.43 \pm 2.69)\times$ & $\text{acc: }55.00 \pm 5.00\,\%$ \\
& & GLAI & $43.00 \pm 7.81$ & $(0.92 \pm 0.48)\times$ & $\text{acc: }48.33 \pm 12.58\,\%$ \\
& & Linear & $74.67 \pm 22.14$ & $(1.39 \pm 0.59)\times$ & $\text{acc: }51.67 \pm 3.82\,\%$ \\

\bottomrule
\end{tabular}
}
\end{table}

Since the task families use different validation objectives, the BVS (Best Validation Score) is the natural metric of each setting: maximum validation accuracy for the supervised and few-shot classification experiments in Families A and C, and minimum validation loss for the self-supervised experiments in Family B. A complementary comparison of the training cost is shown in Figure~\ref{fig:training-time-breakdown}, which breaks down the normalized training time into reduced \ac{mlp} training, conversion, and \ac{glai} training stages relative to the original \ac{mlp}.

A complementary comparison of the training cost is shown in Figure~\ref{fig:training-time-breakdown}, which breaks down the normalized training time into reduced \ac{mlp} training, conversion, and \ac{glai} training stages relative to the original \ac{mlp}.

\begin{figure}[ht]
    \centering
    \includegraphics[width=\textwidth]{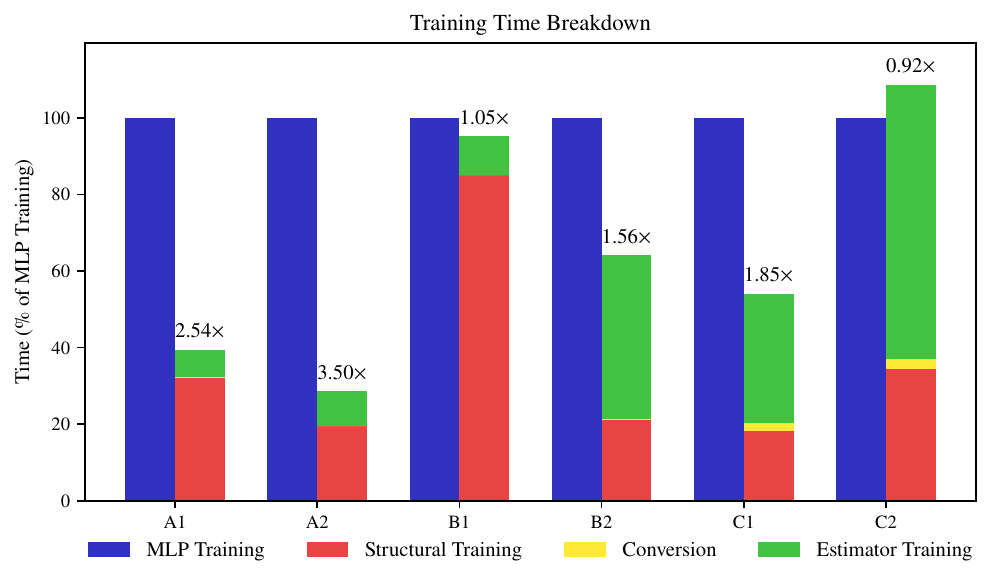}
    \caption{
        Normalized training time breakdown of \ac{glai} relative to the baseline \ac{mlp}. 
        The stacked bars show the contributions of reduced \ac{mlp} training, conversion, and \ac{glai} training. 
        Conversion time is negligible in some cases and may not be visible.
    }
    \label{fig:training-time-breakdown}
\end{figure}

The speedup column measures the ratio between the elapsed training time of the original \ac{mlp} and that of each alternative. Therefore, values above $1\times$ indicate faster training than the original \ac{mlp}, whereas values below $1\times$ indicate slower training. The Best Validation Score (BVS) reports the best validation accuracy for the supervised experiments in families A and C, and the minimum validation loss for the self-supervised experiments in family B.

In family A, \ac{glai} provides a clear improvement in training efficiency while preserving, and even slightly improving, validation accuracy. In A1, \ac{glai} reaches a speedup of $(2.54 \pm 0.52)\times$ over the original \ac{mlp}, while obtaining a higher validation accuracy ($90.75 \pm 0.28\,\%$ versus $88.05 \pm 0.28\,\%$). A similar pattern is observed in A2, where \ac{glai} achieves a speedup of $(3.50 \pm 0.46)\times$ and a validation accuracy of $97.89 \pm 0.16\,\%$, compared with $97.70 \pm 0.01\,\%$ for the original \ac{mlp}. The relatively small deviations in these experiments indicate that the gains are consistent across seeds.

The behavior in family B depends on the specific self-supervised setting. In B1, \ac{glai} and the original \ac{mlp} show comparable training times, with a speedup of $(1.05 \pm 0.40)\times$. Their validation losses are also very close: $(1.70 \pm 0.53)\cdot 10^{-4}$ for \ac{glai} and $(1.63 \pm 0.55)\cdot 10^{-4}$ for the original \ac{mlp}. Given the magnitude of the standard deviations, these differences should be interpreted as practically negligible. Thus, in this setting, \ac{glai} matches the original \ac{mlp} rather than clearly outperforming it. In contrast, B2 shows a stronger advantage for \ac{glai}: it obtains a lower validation loss ($0.159 \pm 0.013$ versus $0.209 \pm 0.010$) together with a speedup of $(1.56 \pm 0.25)\times$, suggesting a more successful replacement in this configuration.

Family C exhibits higher variability, as expected in the few-shot learning setting, where the choice of seed can have a stronger effect on the final outcome. In C1, \ac{glai} obtains a higher mean validation accuracy than the original \ac{mlp} ($96.67 \pm 3.06\,\%$ versus $90.00 \pm 8.00\,\%$) and a speedup of $(1.85 \pm 1.24)\times$. Although this suggests a favorable trend for \ac{glai}, the large standard deviations make the result less conclusive than in family A. In C2, \ac{glai} reports a speedup of $(0.92 \pm 0.48)\times$, which nominally indicates slightly slower training than the original \ac{mlp}; however, the high variability suggests that both methods have comparable training times in practice. The validation accuracies are also noisy in this setting, with \ac{glai} obtaining $48.33 \pm 12.58\,\%$ compared with $52.50 \pm 5.00\,\%$ for the original \ac{mlp}. Therefore, the few-shot experiments should be interpreted as cases where \ac{glai} performs broadly on par with the original \ac{mlp}, with trends that are informative but not as statistically stable as those observed in the fixed-embedding classification experiments.

As a final aggregate view, averaging the six per-experiment speedup ratios yields a mean speedup of $1.92\times$, so \textbf{\ac{glai} trains at nearly twice the speed of the corresponding \ac{mlp} on average}. This value is computed as the mean of ratios, not as the ratio of mean times. The dispersion is asymmetric, with deviations of $+1.25$ and $-0.96$, indicating some variability across configurations but also a longer tail toward higher speedups. Overall, the aggregate timing results show a substantial reduction in training cost across the evaluated settings.

To complement the experiment-wise discussion, we further assess the results through paired hypothesis tests across all configurations. Since the experiments span heterogeneous domains, objectives, and score scales, we do not summarize the comparison through a raw average over BVS values. Instead, we construct paired contrasts over the $(\text{experiment}, \text{seed})$ pairs, yielding $n=18$ paired observations: six experiments and three random seeds per experiment.

For training time, we define the oriented log-ratio
\[
D_{\text{time}}=\log\left(\frac{T_{\ac{mlp}}}{T_{\text{model}}}\right),
\]
so that $D_{\text{time}}>0$ indicates that the compared model converges faster than the original \ac{mlp}. For validation performance, we define an oriented relative score difference. For accuracy-based experiments,
\[
D_{\text{score}}=\frac{S_{\text{model}}-S_{\ac{mlp}}}{|S_{\ac{mlp}}|},
\]
whereas for loss-based experiments,
\[
D_{\text{score}}=\frac{S_{\ac{mlp}}-S_{\text{model}}}{|S_{\ac{mlp}}|}.
\]
Thus, in all cases, $D_{\text{score}}>0$ indicates better validation performance than the original \ac{mlp}. We then apply one-sided paired Wilcoxon signed-rank tests, where the null hypothesis corresponds to no positive paired advantage and the alternative hypothesis corresponds to faster convergence or better validation performance.

\begin{table}[tb]
\centering
\caption{One-sided paired Wilcoxon tests against the original \ac{mlp}, using $n=18$ paired observations. For time, the alternative hypothesis is that the compared model converges faster than the original \ac{mlp}. For BVS, the alternative hypothesis is that the compared model improves validation performance over the original \ac{mlp}.}
\label{tab:wilcoxon_vs_mlp}
\resizebox{0.72\linewidth}{!}{
\begin{tabular}{lcc}
\toprule
Comparison & Time $p$-value & BVS $p$-value \\
\midrule
\ac{glai} vs. \ac{mlp} & $0.0069$ & $0.0276$ \\
Reduced \ac{mlp} vs. \ac{mlp} & $0.4325$ & $0.8597$ \\
Linear vs. \ac{mlp} & $0.0152$ & $0.2475$ \\
\bottomrule
\end{tabular}
}
\end{table}

The results in Table~\ref{tab:wilcoxon_vs_mlp} support the main hypothesis of the proposed method. The comparison between \ac{glai} and the original \ac{mlp} yields a low $p$-value for training time ($p=0.0069$), providing statistical evidence that \ac{glai} converges faster across the paired experiments. The BVS comparison also favors \ac{glai} ($p=0.0276$), although with a weaker margin than in the time comparison. This indicates that the speedup is not obtained at the expense of validation performance; rather, the paired evidence suggests that \ac{glai} also improves the global validation score.

The two auxiliary baselines help qualify this result. The reduced \ac{mlp}, which uses the same reduction factor as \ac{glai} but without the proposed replacement mechanism, does not show evidence of either faster convergence ($p=0.4325$) or improved validation performance ($p=0.8597$). This suggests that the gains obtained by \ac{glai} cannot be attributed merely to reducing the size of the original \ac{mlp}. The linear model, in contrast, does show evidence of faster convergence than the original \ac{mlp} ($p=0.0152$), which is expected given its simpler structure. However, it does not show evidence of improved validation performance ($p=0.2475$). Therefore, while linear heads provide a strong efficiency baseline, they do not match the performance behavior observed for \ac{glai}.

To further contextualize the role of these simplified alternatives, we also perform paired one-sided Wilcoxon tests directly comparing \ac{glai} against the linear head and the reduced \ac{mlp}. In these contrasts, the time statistic is defined as
\[
D_{\text{time}}=\log\left(\frac{T_{\text{baseline}}}{T_{\ac{glai}}}\right),
\]
so that positive values indicate faster convergence for \ac{glai}. The BVS statistic is analogously oriented so that positive values always indicate better validation performance for \ac{glai}.

\begin{table}[tb]
\centering
\caption{One-sided paired Wilcoxon tests comparing \ac{glai} against simplified alternatives, using $n=18$ paired observations. The alternative hypothesis is that \ac{glai} converges faster or obtains better validation performance than the corresponding baseline.}
\label{tab:wilcoxon_glai_baselines}
\resizebox{0.62\linewidth}{!}{
\begin{tabular}{lcc}
\toprule
Comparison & Time $p$-value & BVS $p$-value \\
\midrule
\ac{glai} vs. Linear & $0.0649$ & $0.0564$ \\
\ac{glai} vs. Reduced \ac{mlp} & $0.0038$ & $0.0177$ \\
\bottomrule
\end{tabular}
}
\end{table}

As shown in Table~\ref{tab:wilcoxon_glai_baselines}, \ac{glai} is clearly favored over the reduced \ac{mlp} in both training time ($p=0.0038$) and validation performance ($p=0.0177$). This reinforces the conclusion that the proposed mechanism provides benefits beyond a direct reduction of the \ac{mlp} dimensionality. Against the linear model, the results are more nuanced. For the comparison against the linear model, the results are close to the conventional significance threshold in both dimensions: training time ($p=0.0649$) and BVS ($p=0.0564$). Although these values do not formally cross the $0.05$ level, they consistently favor \ac{glai} in both efficiency and validation performance. This is particularly relevant because the linear head constitutes a highly competitive efficiency-oriented baseline: it is expected to train quickly due to its reduced capacity, but this simplification does not provide the same validation behavior as \ac{glai}. Therefore, the paired tests indicate that \ac{glai} not only remains competitive with the linear alternative in terms of convergence time, but also tends to recover a stronger predictive profile across the heterogeneous experimental settings.

Overall, these paired tests provide a more robust summary than a direct average over heterogeneous experiments. \ac{glai} is the only alternative that shows statistical evidence of both faster convergence and improved validation performance with respect to the original \ac{mlp}. The reduced \ac{mlp} does not reproduce these gains, and the linear model mainly captures the efficiency side without providing a reliable improvement in validation score. Taken together, the results support \ac{glai} as an effective replacement for \acp{mlp}: it reduces the optimization burden while preserving, and in aggregate improving, validation performance across diverse experimental settings.

\section{Conclusion}\label{sec:conclusions}
In this work we have introduced \ac{glai}, a new architectural block that revisits the role of \acp{mlp} by separating structural from quantitative knowledge. Previous analyses \citep{duato} and our own theoretical and practical results (\ref{subsec:path_distance}) indicate that activation patterns stabilize significantly earlier than weights. We turn this observation into a training principle through the \ac{glai} framework: once structural knowledge stabilizes, it is fixed, and training proceeds only on the quantitative component, which naturally collapses to a single-layer linear model.

Our experimental results show that this shift translates into consistent practical gains. \ac{glai} models replace conventional \ac{mlp} heads while maintaining accuracy, yet they require slightly less than $60\%$ of the training time. Beyond speed, this reduction has tangible implications in terms of computational cost and energy use, thereby contributing to a more sustainable deployment of \ac{dl} models.

Overall, \ac{glai} emerges as an efficient alternative to conventional \acp{mlp}, grounded in both theoretical arguments and empirical evidence across diverse setups. By demonstrating that structural knowledge can be fixed early without loss of predictive power, this work opens the door to a broader line of research on path-based formulations. We expect that extending these ideas to more complex architectures may provide further insights into the interplay between expressivity, efficiency, and sustainability in modern \ac{dl}.

\section{Limitations and Future Work}

The current implementation of \ac{glai} is primarily intended for shallow \ac{mlp} modules, such as the heads and projection layers evaluated in this work. Its main practical limitation comes from the growth of the path space with depth: when the replaced \ac{mlp} contains many hidden layers, the estimator may require very aggressive pruning, which can reduce stability or limit the potential efficiency gains. A second limitation concerns the particular reduction strategy adopted in this paper, which uniformly scales the hidden layers by a factor $\rho$ and assumes $\rho \cdot n_L \ge n_{L+1}$ in order to avoid a bottleneck before the output layer. This constraint is not intrinsic to the \ac{glai} formulation itself, but to the simple compression rule used here to construct the reduced \ac{mlp}. If this condition is not satisfied in a given architecture, a different reduction or pruning strategy would be required. In many practical head and projection-layer settings, however, these conditions can be satisfied without difficulty, which is why the present formulation is sufficient for the scenarios studied in this work.

This study has focused on frozen-backbone scenarios, where the head is the main adaptation bottleneck. However, the architectural principle is not restricted to this setting. A major direction for future work is to investigate whether \ac{glai} can be extended to large-scale architectures such as transformers. This extension is non-trivial: transformer models rely on tightly coupled attention, normalization, residual pathways, and feed-forward blocks, and replacing intermediate \acp{mlp} without disrupting these interactions requires careful theoretical and empirical analysis. Nevertheless, if these challenges can be addressed, such an extension could contribute to more efficient training and inference while also opening opportunities for interpretability in large-scale models. On the other hand, defining a reliable stopping criterion based on structural knowledge remains an open problem. Although it is known that activation patterns tend to stabilize earlier than weights, translating this behavior into a robust and task-independent proxy is still unresolved. Developing such a criterion, enabling early stopping driven by structural maturity rather than validation performance, is part of our ongoing work. Lastly, we are also exploring the use of synthetic weights and structural patterns to better characterize the interaction between structure and parameters.

Another relevant direction is to study how \ac{glai} interacts with other efficiency-oriented techniques. The comparison in this work intentionally focuses on the vanilla \ac{mlp} and vanilla \ac{glai} formulations, in order to isolate the effect of the proposed architectural reformulation. However, \ac{glai} is largely orthogonal to several acceleration, compression, and distillation strategies. In principle, methods such as distillation, improved training schedules, specialized regularization, or further compression techniques could be applied on top of a \ac{glai} module, just as they can be applied to conventional \acp{mlp}. Exploring these combinations is therefore a natural direction for future work.

Finally, a deeper understanding of the optimization geometry of \ac{glai} remains necessary. The estimator does not optimize the same parameter space as a conventional \ac{mlp}: after the structural component has been fixed, training takes place over path weights, leading to a different loss landscape. In our experiments, this change in parameterization also affects the behavior of regularization and optimizers, and suggests that the \ac{glai} loss can be flatter in some regimes. This may make optimization less sensitive to some directions of descent and may explain why certain hyperparameter choices, such as stronger weight decay, are more beneficial for \ac{glai} than for the original \ac{mlp}. A systematic study of this parameter space, the geometry of the associated loss function, and the dynamics of different optimizers is left for future work.

\subsubsection*{Acknowledgments}
This research was funded by the projects PID2023-146569NB-C21 and PID2023-146569NB-C22 supported by MICIU/AEI/10.13039/501100011033 and ERDF/UE. Jose I. Mestre was supported by the predoctoral grant ACIF/2021/281 of the \emph{Generalitat Valenciana}. Alberto Fernández-Hernández was supported by the predoctoral grant PREP2023-001826 supported by MICIU/AEI/10.13039/501100011033 and ESF+. Cristian Pérez-Corral received support from the \textit{Conselleria de Educación, Cultura, Universidades y Empleo} (reference CIACIF/2024/412) through the European Social Fund Plus 2021–2027 (FSE+) program of the \textit{Comunitat Valenciana}. Manuel F. Dolz was supported by grant {\small CNS2025-165098} funded by {\small MICIU/AEI/10.13039/501100011033} and by the Plan Gen--T grant {\small CIDEXG/2022/013} of the \emph{Generalitat Valenciana}.

\bibliographystyle{elsarticle-num-names} 
\bibliography{references}



\end{document}